%% file: main.tex
\documentclass[letterpaper]{article}
\usepackage{uai2019}
\usepackage[margin=1in]{geometry}

\newif\ifsup
\suptrue

\usepackage{natbib}
\usepackage{amsmath}
\usepackage{amsthm}
\usepackage{algorithm}
\usepackage{mdframed}
\usepackage{amssymb}
\usepackage{mathtools}
\usepackage{color}
\usepackage{dsfont}
\usepackage{bm}
\usepackage{pgfplots}
\usepackage[textsize=tiny]{todonotes}

  \definecolor{dkblue}{cmyk}{1,.54,.04,.19} 
  \usepackage[pagebackref]{hyperref} 
  \usepackage{calc}
  \hypersetup{
      bookmarks=true,         
      unicode=false,          
      pdftoolbar=true,        
      pdfmenubar=true,        
      pdffitwindow=false,     
      pdfstartview={FitH},    
      pdftitle={Adversarial Bandits},    
      pdfauthor={},     
      pdfsubject={Bandits},   
      pdfcreator={pdflatex},   
      pdfproducer={Producer}, 
      pdfkeywords={bandits} {statistics} {machine learning}, 
      pdfnewwindow=true,      
      colorlinks=true,        
      linkcolor=dkblue,       
      citecolor=dkblue,       
      filecolor=dkblue,       
      urlcolor=dkblue,        
  }

\usepackage[capitalise]{cleveref}
\usepackage{times}

\title{On First-Order Bounds, Variance and Gap-Dependent Bounds for Adversarial Bandits}

\author{Roman Pogodin \\
Gatsby Computational Neuroscience Unit\\
University College London, London, UK\\
\href{mailto://roman.pogodin.17@ucl.ac.uk}{roman.pogodin.17@ucl.ac.uk}
\And
Tor Lattimore \\
DeepMind\\
London, UK\\
\href{mailto://tor.lattimore@gmail.com}{tor.lattimore@gmail.com}
}
\input{header}

\begin{document}

\maketitle

\begin{abstract}
We make three contributions to the theory of $k$-armed adversarial bandits. First, we prove a first-order bound for a modified variant of the INF 
strategy by \cite{AB09}, without sacrificing worst case optimality or modifying the loss estimators. Second, we provide a variance analysis for algorithms based on follow the regularised leader, showing that without
adaptation the variance of the regret is typically $\Omega(n^2)$ where $n$ is the horizon.
Finally, we study bounds that depend on the degree of separation of the arms, generalising the results by \cite{CK15c} from the stochastic setting to the adversarial
and improving the result of \cite{SS14} by a factor of $\log(n) / \log\log(n)$.
\end{abstract}

\section{INTRODUCTION}

\input{intro}

\section{NOTATION}

\input{notation}

\section{FOLLOW THE REGULARISED LEADER}

\input{algorithm}

\section{FIRST ORDER BOUNDS}

\input{fo}

\section{VARIANCE OF THE REGRET}

\input{variance}

\section{LINEARLY SEPARABLE BANDITS}

\input{gap}

\section{OPEN QUESTIONS}

\input{discussion}

\subsubsection*{Acknowledgements}

\input{acknowledgements}

\bibliographystyle{plainnat}
\bibliography{all}

\ifsup
\appendix

\section{TECHNICAL INEQUALITIES}\label{app:tech_inequalities}

\input{tech}

\section{PROOF OF COROLLARY~\ref{corollary:hybrid-known-horizon}}\label{app:fo-corollary}

\input{fo-appendix}

\section{PROOF OF THEOREM~\ref{thm:asy}}\label{app:asy}

\input{gap-app}

\section{PROOF OF THEOREM~\ref{thm:linear}}\label{app:linear}

\input{L1-gap}

\fi

\end{document}

%% file: header.tex
\newcommand{\half}{\frac{1}{2}}

\DeclareMathOperator*{\argmin}{arg\,min}

\DeclarePairedDelimiterX{\infdivx}[2]{(}{)}{%
  #1\;\delimsize\|\;#2%
}

\DeclarePairedDelimiter{\norm}{\lVert}{\rVert}

\newcommand{\cO}{O }
\newcommand{\cG}{\mathcal G}
\renewcommand{\mid}{\,|\,}
\newcommand{\N}{\mathbb N}
\newcommand{\Var}{\operatorname{Var}}
\newcommand{\E}{\mathbb E}
\newcommand{\cF}{\mathcal F}
\newcommand{\ones}{\bm 1}
\newcommand{\cA}{\mathcal A}
\newcommand{\interior}{\operatorname{interior}}
\newcommand{\dom}{\operatorname{dom}}
\newcommand{\R}{\mathbb R}

\newcommand{\ceil}[1]{\left\lceil{#1}\right\rceil}

\newcommand{\Prob}[1]{\mathbb{P}\left(#1\right)}

\newcommand{\EE}[1]{\mathbb{E}\left[#1\right]}

\newcommand{\expect}{\mathbb{E}\,}

\newcommand{\set}[1]{\left\{ #1 \right\}}

\newcommand{\brackets}[1]{\left(#1\right)}

\newcommand{\diag}{\,\mathrm{diag\,}}

\newcommand{\inv}{^{-1}}

\newcommand{\dotprod}[2]{\left\langle #1, #2\right\rangle}
\newcommand{\ip}[1]{\langle #1\rangle}
\newcommand{\one}[1]{\mathds{1}\left\{#1\right\}}

\theoremstyle{plain}
\newtheorem{theorem}{Theorem}[section]
\newtheorem{lemma}[theorem]{Lemma}

\newtheorem{corollary}[theorem]{Corollary}

\theoremstyle{definition}

\newtheorem{assumption}[theorem]{Assumption}
\newtheorem{remark}[theorem]{Remark}
\theoremstyle{remark}

%% file: intro.tex
The $k$-armed adversarial bandit is a sequential game played over $n$ rounds.
At the start of the game the adversary secretly chooses a sequence
of losses $(\ell_t)_{t=1}^n$ with $\ell_t \in [0,1]^k$. In each round $t$ the learner chooses a distribution $P_t$ over the actions $[k] = \{1,2,\ldots,k\}$. 
An action $A_t \in [k]$ is sampled from $P_t$ and the learner observes the loss $\ell_{tA_t}$.
Like prior work we focus on controlling the regret, which is
\begin{align*}
\hat R_n = \max_{i \in [k]} \sum_{t=1}^n \left(\ell_{tA_t} - \ell_{ti}\right)\,.
\end{align*}
This quantity is a random variable, so the standard objective is to bound $\hat R_n$ with high probability or its expectation: $R_n = \E[\hat R_n]$.

We make three contributions, with the common objective of furthering our understanding of the application of follow the regularised leader (FTRL) to adversarial bandit problems.
Our first contribution is a modification of the INF policy by \cite{AB09} in order to prove first-order bounds (i.e. in terms of the loss of the best action) without sacrificing minimax optimality.
Then we turn our attention to the variance of algorithms based on FTRL. Here we prove that using the standard importance-weighted estimators and a large class of potentials leads
to a variance of $\Omega(n^2)$, which is the worst possible for bounded losses.
Finally, we investigate the asymptotic performance of algorithms when there is a linear separation between the losses of the arms. We improve the result by \cite{SS14} by a factor of $\log(n) / \log \log(n)$ and 
generalise known results in the stochastic
setting by \cite{CK15c} to the adversarial one by constructing an algorithm for which the regret grows arbitrarily slowly almost surely.

\paragraph{Related work}
The literature on adversarial bandits is enormous. See the books by \cite{BC12} and \cite{LS19bandit-book} for a comprehensive account.
The common thread in the three components of our analysis is adaptivity for algorithms based on follow the regularised leader.
The INF policy that underlies much of our analysis was introduced by \cite{AB09}. The connection to mirror descent and follow the regularised leader came
later \citep{AB10,BC12}, which greatly simplified the analysis. The principle justification for introducing this algorithm was to prove bounds on the
minimax regret. Remarkably, it was recently shown that by introducing a non-adaptive decaying learning rate, the algorithm retains minimax
optimality while simultaneously achieving a near-optimal logarithmic regret in the stochastic setting \citep{ZiSe18}. Despite its simplicity, the algorithm
improves on the state-of-the-art for this problem \cite{BuSli12}, \cite{SS14}, \cite{SeLu17}. See also the extension to the combinatorial semibandit setting \citep{ZLW19}.
First-order bounds for bandits were first given by \cite{AAGO06}, who analysed a modification of Exp3 \citep{ACFS95}.  As far as we know, previous algorithms with first
order bounds have not been minimax optimal ($R_n=O(\sqrt{kn})$): the recent work by \cite{Neu15b} achieved $O(\sqrt{kn(\log(k)+1)})$ expected regret, and \citep{CL18} had a $O(\sqrt{kn\log n})$ bound. Both papers used the idea of an adaptive learning rate similar to our analysis. In the setting of gains rather than losses \cite{AB10} have shown that by introducing biased estimators it is possible to
prove a bound of $O(\sqrt{k G^*})$ where $G^*$ is the maximum gain. Although it is not obvious, we suspect the same idea could be applied in our setting. 
We find it interesting nevertheless that the same affect is possible without modifying the loss estimators.  
The aforementioned work also assumes knowledge of $G^*$. Possibly our adaptive learning rates could be used to make this algorithm anytime without a doubling trick.

Although it is well known that straightforward applications of follow the regularised leader or mirror descent with importance-weighted estimators leads to poor concentration of the regret,
we suspect the severity of the situation is not widely appreciated. As far as we know, the quadratic variance of Exp3 was only derived recently \citep[\S11]{LS19bandit-book}.
There are, however, a number of works modifying the importance-weighted estimators to prove high probability bounds \cite{ACFS95,AR09,Neu15} with matching lower bounds by \cite{GL16}.
Finally, we note there are many kinds of adaptivity beyond first-order bounds. For example sparsity and variance \citep[and others]{BCL17,HK11b}.

%% file: notation.tex
Given a vector $x \in \R^d$ let $\diag(x) \in \R^{d\times d}$ be the diagonal matrix with $x$ along the diagonal. 
The interior of a topological space $X$ is $\interior(X)$ and its boundary is $\partial X$.
The standard basis vectors are $e_1,\ldots,e_d$.
The $(d-1)$-dimensional probability simplex is $\Delta^{d-1} = \{x \in [0,1]^d : \norm{x}_1 = 1\}$.
A convex function $F : \R^d \to \R \cup \{\infty\}$ has domain $\dom(F) = \{x \in \R^d : F(x) \neq \infty\}$.
The Bregman divergence with respect to a differentiable $F$ is a function $D_F : \dom(F) \times \dom(F) \to [0, \infty]$ defined by
$D_F(x, y) = F(x) - F(y) - \left\langle\nabla F(y), x-y\right\rangle$. The Fenchel dual of $F$ is $F^* : \R^d \to \R \cup \{\infty\}$ defined by $F^*(u) = \sup_{x \in \R^d} \ip{x, u} - F(x)$.

There are $k$ arms and the horizon is $n$, which may or may not be known.
The losses are $(\ell_t)_{t=1}^n$ with $\ell_t \in [0,1]^k$.
We let $L_t = \sum_{s=1}^t \ell_s$.
The importance-weighted estimator of $\ell_t$ is $\hat \ell_t$ defined by
$\hat \ell_{ti} = \one{A_t = i} \ell_{ti}/ P_{ti}$.
All algorithms proposed here ensure that $P_{ti} > 0$ for all $t$ and $i$, so this quantity is
always well defined. Let $\hat L_t = \sum_{s=1}^t \hat \ell_s$.
Expectations are with respect to the randomness in the actions $(A_t)_{t=1}^n$.
Of course the learner can only choose $P_t$ based on information available at the start of round $t$.
Let $\cF_t = \sigma(A_1,\ldots,A_t)$.
Then $P_t$ is $\cF_{t-1}$-measurable. 
Let $A_{ti} = \one{A_t = i}$ and $T_i(t) = \sum_{s=1}^t A_{si}$ be the number of times arm $i$ is played in the first $t$ rounds.
Our standing assumption is that the first arm is optimal. All our algorithms are symmetric, so this is purely for notational convenience. 

\begin{assumption}
$L_{t1} = \min_{i \in [k]} L_{ti}$.
\end{assumption}

%% file: algorithm.tex
Follow the regularized leader (FTRL) is a popular tool for online optimization \citep{Sha07,Haz16}.
The basic algorithm depends on a sequence of potential functions $(F_t)_{t=1}^\infty$ where
$F_t : \R^d \to \R \cup \{\infty\}$ is convex and $\dom(F_t) \cap \Delta^{k-1} \neq \emptyset$.
In each round the algorithm chooses the distribution
\begin{align*}
P_t = \argmin_{p \in \Delta^{k-1}} \ip{p, \hat L_{t-1}} + F_t(p)\,,
\end{align*}
which we assume exists.
The action $A_t \in [k]$ is sampled from $P_t$.
In many applications $F_t = F$ is chosen in a time independent way, with examples
given in \cref{tab:pot}. This has the disadvantage that $F$ must be chosen in advance in a way that
depends on the horizon, which may be unknown. This weakness can be overcome by choosing $F_t = F / \eta_t$  
where $(\eta_t)_{t=1}^\infty$ is a sequence of learning rates, which may be chosen in advance or adaptively in a 
data-dependent way.

A modification that will prove useful is to let $(\cA_t)_{t=1}^\infty$ be a sequence of subsets of $\Delta^{k-1}$ and
define
\begin{align*}
P_t = \argmin_{p \in \cA_t} \ip{p, \hat L_{t-1}} + F_t(p)\,.
\end{align*}
The restriction to a subset of $\Delta^{k-1}$ can be useful to control the gradients of $F_t(P_t)$, which is sometimes crucial.
\begin{table}
\renewcommand{\arraystretch}{1.6}
\centering
\small
\begin{tabular}{|lll|}
\hline 
\textbf{Potential} & \textbf{Definition} & \textbf{Alg.} \\
Negentropy & $\frac{1}{\eta} \sum_{i=1}^k p_i (\log(p_i) - 1)$ & Exp3 \\
1/2-Tsallis &  $-\frac{2}{\eta} \sum_{i=1}^k \sqrt{p_i}$ & INF \\ 
Log barrier & $-\frac{1}{\eta}\sum_{i=1}^k \log(p_i)$ & \\ \hline
\end{tabular}
\caption{Common potential functions}\label{tab:pot}
\end{table}
The following theorem provides a generic bound for FTRL with changing potentials and constraint sets.
The result is reminiscent of many previous bounds for FTRL, but a reference for this result seems elusive.
Most related is the generic analysis by \cite{JGS17}, which also provides the most comprehensive literature summary.

\begin{theorem}\label{thm:ftrl}
  Assume $\cA_1 \subseteq \cdots \subseteq \cA_{n+1} \subseteq \Delta^{k-1}$ and
 $(F_t)_{t=1}^{n+1}$ is a sequence of convex functions with $\dom(F_t) \cap \cA_t \neq \emptyset$ for all $t$.
 Define 
  \begin{align*}
  d_t &= \max_{y \in \cA_{t+1}} \min_{x \in \cA_t} \norm{x - y}_1\,, \quad
  g_t = \sup_{x \in \cA_t} \norm{\nabla F_t(x)}_\infty \\
    &\qquad\text {and} \qquad v_n = \sum_{t=1}^n d_t (g_t + (t-1))\,.
  \end{align*}
  Then the regret of FTRL is bounded by
    \begin{align*}
        &R_n \leq v_n + \EE{\sum_{t=1}^n \dotprod{P_t - P_{t+1}}{\hat \ell_t} - D_{F_t}(P_{t+1}, P_t)} \\
        &\qquad+ \EE{\min_{p \in \cA_{n+1}} \brackets{F_{n+1}(p) + n\| p - e_1 \|_1} - F_1(P_1)} \\
        &\qquad+ \EE{\sum_{t=1}^n\brackets{F_{t}(P_{t+1}) - F_{t+1}(P_{t+1})}} \,.
    \end{align*}
\end{theorem}

\begin{proof}
  Let $p \in \cA_{n+1}$. Using the fact that $\hat \ell_t$ is unbiased,
    \begin{align*}
        R_n
        &= \EE{\sum_{t=1}^n \ip{P_t - e_1, \hat \ell_t}} \\ 
        &= \EE{\sum_{t=1}^n \ip{P_t - p, \hat \ell_t}} + \EE{\sum_{t=1}^n \ip{p - e_1, \ell_t}}\,.
    \end{align*}
    The second sum is the approximation error, and by Holder's inequality,
    \[
        \sum_{t=1}^n \ip{p - e_1, \ell_t} \leq \| p - e_1 \|_1\sum_{t=1}^{n} \norm{ \ell_t}_\infty \leq n\| p - e_1 \|_1 \,.
    \]
    Therefore, 
    \begin{align*}
        &R_n 
        \leq \EE{\sum_{t=1}^n \ip{P_t - P_{t+1}, \hat \ell_t} + \sum_{t=1}^n \ip{P_{t+1} - p, \hat \ell_t}} 
        \\ & \qquad\qquad + n \|  p - e_1 \|_1 \,. 
    \end{align*}
    Let $\Phi_t(q) = F_t(q) + \sum_{s=1}^{t-1}\ip{q, \hat \ell_s}$, which is chosen so that $P_t = \argmin_{q \in \cA_t} \Phi_t(q)$.
    Then the second sum in the above display equals
    \begin{align*}
        &\sum_{t=1}^n\brackets{\Phi_{t+1}(P_{t+1}) - \Phi_{t}(P_{t+1}) - F_{t+1}(P_{t+1}) + F_{t}(P_{t+1})} \\
        &\qquad \qquad - \Phi_{n+1}(p) + F_{n+1}(p) \\ 
        &= \sum_{t=1}^n\brackets{\Phi_{t}(P_t) - \Phi_{t}(P_{t+1})} \\
        &\qquad\qquad + \Phi_{n+1}(P_{n+1}) - \Phi_1(P_1) - \Phi_{n+1}(p) \\ 
        &\qquad\qquad + F_{n+1}(p) + \sum_{t=1}^n\brackets{F_{t}(P_{t+1}) - F_{t+1}(P_{t+1})}
    \end{align*}
    We can rewrite the $\Phi$-differences as
    \begin{align*}
        &\Phi_t(P_t) - \Phi_t(P_{t+1}) \\
         & \qquad= -D_{\Phi_t}(P_{t+1}, P_t) - \ip{\nabla \Phi_t(P_t), P_{t+1} - P_t}\,.
    \end{align*}
    Let $\delta_t = P_{t+1} - \argmin_{q\in \cA_t}\| q - P_{t+1}\|_1$. Then due to first-order optimality condition for $P_{t}$ on $\cA_t$,
    \begin{equation*}
        \EE{\ip{\nabla \Phi_t(P_t), (P_{t+1}-\delta_t) - P_t}} \geq 0,
    \end{equation*}
    therefore
    \begin{align*}
    &\EE{\ip{\nabla \Phi_t(P_t), P_{t+1} - P_t}}
    \geq \EE{\ip{\nabla \Phi_t(P_t), \delta_t}} \\
    &\qquad\geq \EE{\ip{\nabla F_t(P_t), \delta_t} + \sum_{s=1}^{t-1} \ip{\hat \ell_s, \delta_t}} \\
    &\qquad\geq -\EE{\norm{\delta_t}_1 \left(\norm{\nabla F_t(P_t)}_\infty + \sum_{s=1}^{t-1} \norm{ \ell_s}_\infty\right)} \\
    &\qquad\geq -d_t g_t - d_t \sum_{s=1}^{t-1} \norm{ \ell_s}_\infty \geq -d_t (g_t + (t-1))\,,
    \end{align*}
    where we used Holder's inequality, the definitions of $d_t$ and $g_t$, non-negativity of $\hat \ell_s$ and that $\expect\hat \ell_s=\ell_s\in[0,1]$.
    It follows that 
    \begin{align*}
        \Phi_t(P_t) - \Phi_t(P_{t+1}) 
        \leq 
        d_t(g_t + k(t-1)) - D_{\Phi_t}(P_{t+1}, P_t) \,.
    \end{align*}
    Since $p \in \cA_{n+1}$ and $P_{n+1}$ is the minimiser of $\Phi_{n+1}$ in $\cA_{n+1}$, we have
    \begin{align*}
        \Phi_{n+1}(P_{n+1}) - \Phi_{n+1}(p) \leq 0\,.
    \end{align*}
    Finally, noting that $\Phi_1 = F_1$ and $D_{\Phi_t}(P_{t+1}, P_t) = D_{F_t}(P_{t+1}, P_t)$ we obtain
    \begin{align*}
        R_n &\leq n \| p - e_1 \|_1 + \sum_{t=1}^n d_t(g_t + (t-1)) \\
        &\quad+ \EE{\sum_{t=1}^n \ip{P_t - P_{t+1}, \hat \ell_t} - D_{F_t}(P_{t+1}, P_t)} \\
        &\quad+ \EE{F_{n+1}(p) - F_1(P_1)} \\
        &\quad+ \EE{\sum_{t=1}^n\brackets{F_{t}(P_{t+1}) - F_{t+1}(P_{t+1})}},
    \end{align*}
    from which the statement follows.
\end{proof}

%% file: fo.tex
We now introduce the modification of the INF strategy, which
takes inspiration from \cite{CL18,ZiSe18,ZLW19}.
The new algorithm plays on the `chopped' simplex, with the magnitude of the cut dependent on the round,
\begin{equation}
    \cA_t = \Delta^{k-1} \cap [1/t,1]^k\,.
    \label{eq:chopped_simplex}
\end{equation}
Then for a convex potential $f_t(p)$ with $\dom(f_t)^k\cap \Delta^{k-1} \neq \emptyset$ define a potential
\begin{align}
F_t(p) = \frac{1}{\eta_t}\sum_{i=1}^k f_t(p_i)\,,
\label{eq:potential}
\end{align}
where the learning rate $\eta_t$ is given by 
\begin{align}
\eta_t &= \frac{\eta_0}{\sqrt{1 + \sum_{s=1}^{t-1} \hat \ell_{sA_s}^2 (\nabla^2(f_s)(P_{sA_s}))\inv }}\,,
\label{eq:eta}
\end{align}
where $\eta_0$ is positive constant to be tuned later.

The Hessian of the potential plays a fundamental role in the regret, simplifying the derivation of a generic first-order bound:
\begin{theorem}\label{theorem:hessian-first-order}
Suppose that $\nabla^2 f_t$ is decreasing on $(0,1)$ and there
exist $B, C \geq 0$ such that
\[
\frac{1}{p^2 \nabla^2 f_t(p)} \leq B, \quad \E\left[\frac{1}{P_{tA_t}^2 \nabla^2 f_t(P_{tA_t})}\right] \leq C\,,
\]
for all $p \in (0,1)$ and $t \in [n]$. Assume additionally that
there exist a non-negative constant $h_1$ and a non-negative function $h_2(n)$ such that
\[
\begin{split}
        &v_n + \min_{p \in \cA_{n+1}} \brackets{F_{n+1}(p) + \| p - e_1 \|_1 n} - F_1(P_1) \\ &\quad+\sum_{t=1}^n\brackets{F_{t}(P_{t+1}) - F_{t+1}(P_{t+1})} 
        \leq \frac{h_1}{\eta_{n+1}} + h_2(n)\,,
\end{split}
\]
almost surely. Then the expected regret of FTRL with $\eta_0 = \sqrt{h_1} / 2^{1/4}$ simultaneously satisfies
\[
\begin{split}
        &R_n \leq \frac{\sqrt{h_1}}{2^{5/4}} B + h_2(n) + 2\sqrt{2} B h_1 +2^{7/4}\sqrt{h_1}\\
        &\quad\times \sqrt{1 + \frac{B L_{n1}}{2} + \frac{B h_2(n)}{2} + B^2\brackets{\frac{\sqrt{h_1}}{2^{9/4}} + \frac{h_1}{\sqrt{2}}}}\,,\\
        &R_n \leq \frac{\sqrt{h_1}}{2^{5/4}} B + h_2(n) + 2^{7/4}\sqrt{h_1} \sqrt{1 + \frac{Cn}{2}}\,.
\end{split}
\]
\end{theorem}

\begin{remark}
    $h_1$ and $h_2(n)$ reflect the approximation error, non-stationarity of the potential $f_t$ and how sensitive it is to the changes in $\cA_t$. In a simple case with $\cA_t = \cA,\ f_t = f$ for all $t$, this is a standard bound for the sum of the potential differences. $h_1$ can be a function of $n$ when the horizon $n$ is known, as we choose the learning rate based on it.
\end{remark}

As an application of this general first-order result, we derive a worst-case optimal bound for a carefully chosen mixture of the INF regularizer and the log-barrier:
\begin{corollary}\label{corollary:hybrid-first-order}
    For $\eta_0 = k^{1/4}\sqrt{\frac{13}{3\sqrt{2}} + \frac{3}{\sqrt{2}q}}$ and 
    \[
    f_t(p) = -2\sqrt{p} - \frac{\log p}{\sqrt{k}\log^{1+q} \max\{3, t\}}
    \]
    and any $q > 0$ and $n\geq 3$, the regret grows with $n$ as
    \[
        R_n = \cO\left(\sqrt{k L_{n1}\log^{1+q}(n) + k^{2}\log^{2(1+q)} n} + k\log^{1+q}(n)\right)\,,
    \]
    with some constants proportional to $1/q$.
\end{corollary} 

\begin{corollary}\label{corollary:hybrid-first-order-q}
    For $q=1$, $\eta_0 = k^{1/4}\sqrt{22 / (3\sqrt{2})}$,
    \begin{align*}
        &R_n \leq 19 k^2+22 k \log ^2(n)+2 k \log (n) + 6.5 \log (n)\\
        &\quad\times \sqrt{k L_{n1} + 19 k^3 +2 k^2 \log (n)+11.2 k^2 \log ^2(n)}\,.
    \end{align*}
    In the worst-case scenario the regret satisfies
    \[
        \limsup_{n\rightarrow \infty} \frac{R_n}{\sqrt{kn}} \leq 9.2.
    \]
\end{corollary} 

\begin{corollary}\label{corollary:hybrid-known-horizon}
    If the horizon $n\geq 3$ is known in advance, using $\mathcal{A}_t = \Delta^{k-1} \cap [1/n, 1]^k$, $\eta_0 = k^{1/4}\sqrt{3}/2^{1/4}$ and 
    \[
    f_t(p) = -2\sqrt{p} - \frac{\log p}{\sqrt{k}\log n}
    \]
    results in
    \begin{align*}
        &R_n \leq k + 9.1 k \log n \\
        &\qquad + 4.2  \sqrt{k L_{n1}\log (n) + 2\sqrt{k} + 6 k^2 \log ^2(n)}\,,\\
        &\qquad\qquad\qquad\limsup_{n\rightarrow \infty} \frac{R_n}{\sqrt{kn}} \leq 5.9\,.
    \end{align*}
\end{corollary}
The proof of the last corollary simply repeats previous statements, also using the stationarity of the constraint set and $f_t(p)$. See
\ifsup
\cref{app:fo-corollary} 
\else
the supplementary material 
\fi
for more details.

\begin{remark}
    \cref{theorem:hessian-first-order} with known $n$ reproduces the result of \citep{CL18} (note that they used a slightly different algorithm and the learning rate schedule): for the log-barrier potential $f_t(p)=-\log p$ we have $B=C=1$ and $h_1(n)\propto k\log n$, such that the worst-case regret is $R_n=O(\sqrt{kn\log n})$.
\end{remark}

The proof of \cref{theorem:hessian-first-order} follows from \cref{thm:ftrl} and the following lemmas:
\begin{lemma}\label{lemma:dual-norm-bound}
    For a potential of the form \cref{eq:potential} with $\nabla^2 f_t(p)$ that is monotonically decreasing on $p\in (0, 1)$,  
    \begin{align*}
        &\sum_{t=1}^n \dotprod{P_t - P_{t+1}}{\hat \ell_t} - D_{F_t}(P_{t+1}, P_t) \\
        &\qquad\qquad\leq\sum_{t=1}^n\frac{\eta_t}{2}\frac{\ell_{tA_t}^2}{P_{tA_t}^2\nabla^2 f_t(P_{tA_t})}\,.
    \end{align*}
\end{lemma}
\begin{proof}
        Let $t \in [n]$ and suppose that $P_{t+1,A_t} > P_{tA_t}$.
    Then using the fact that the loss estimators and Bregman divergence are non-negative,
    \begin{align*}
    &\ip{P_t - P_{t+1}, \hat \ell_t} - D_{F_t}(P_{t+1}, P_t) 
    \leq \ip{P_t - P_{t+1}, \hat \ell_t} \\
    &\qquad= (P_{tA_t} - P_{t+1,A_t}) \hat \ell_{tA_t} \leq 0\,.
    \end{align*}
    Now suppose that $P_{t+1,A_t} \leq P_{tA_t}$.
    By \citep[Theorem 26.5]{LS19bandit-book}, 
    \begin{align*}
        \ip{P_t - P_{t+1}, \hat \ell_t} - D_{F_t}(P_{t+1}, P_t) \leq \frac{1}{2}\|\hat \ell_t\|^2_{(\nabla^2F_t(z))\inv}\,,
    \end{align*}
    where $z = \alpha P_t + (1 - \alpha)P_{t+1}$ for some $\alpha \in [0,1]$. 
    By definition $\nabla^2 F_t(z) = \diag(\nabla^2 f_t(z)) / \eta_t$ and since $\hat \ell_{ti} = 0$ for $i \neq A_t$,
    \begin{align*}
    \frac{1}{2} \norm{\hat \ell_t}^2_{\nabla^2 F_t(z)^{-1}} 
    &= \frac{\eta_t \hat \ell_{tA_t}^2}{2\nabla^2 f_t(z_{A_t})} 
    \leq \frac{\eta_t \hat \ell_{tA_t}^2}{2 \nabla^2 f_t(P_{tA_t})}\,,
    \end{align*}
    where we used the fact that $z_{A_t} \leq P_{tA_t}$ and that $\nabla^2 f_t(p)$ is decreasing. The result follows by substituting the definition of $\hat \ell_{tA_t}$ and summing over $t \in [n]$.
\end{proof}

\begin{lemma}
    \label{lemma:sqrt_sum_bound}
    Let $(x_t)_{t=1}^n$ be a sequence with $x_t \in [0, B]$ for all $t$. Then
    \begin{align*}
        \sum_{t=1}^n \frac{x_t}{\sqrt{1 + \sum_{s=1}^{t-1} x_s}} \leq 4\sqrt{1 + \frac{1}{2}\sum_{t=1}^n x_t} + B\,.
    \end{align*}
\end{lemma}

The proof follows from a comparison to an integral and is given in
\ifsup
\cref{app:tech_inequalities}.
\else
the supplementary material.
\fi

\begin{proof}[Proof of \cref{theorem:hessian-first-order}]
    Using the result of \cref{thm:ftrl}, \cref{lemma:dual-norm-bound} and the assumption on the difference in the potentials, we have
    \begin{align*}
        R_n \,&\leq h_2(n) + \EE{\frac{h_1}{\eta_{n+1}}+\sum_{t=1}^n \frac{\eta_t}{2}\frac{\ell_{tA_t}^2}{P_{tA_t}^2\nabla^2 f_t(P_{tA_t})}}\,.
    \end{align*}
    As $\ell_{tA_t} \leq 1$, we can apply \cref{lemma:sqrt_sum_bound} with
    \[
        x_t = \frac{\ell_{tA_t}^2}{P_{tA_t}^2\nabla^2 f_t(P_{tA_t})} \leq B \,.
    \]
    It follows that $\eta_t = \eta_0 / \sqrt{1 + \sum_{s=1}^{t-1} x_s}$, and thus
    \begin{align*}
        &\sum_{t=1}^n \frac{\eta_t}{2}\frac{\ell_{tA_t}^2}{P_{tA_t}^2\nabla^2 f_t(P_{tA_t})} \leq \\
        &\qquad 2\eta_0\sqrt{1 + \half \sum_{t=1}^n \frac{\ell_{tA_t}^2}{P_{tA_t}^2\nabla^2 f_t(P_{tA_t})}} + \frac{\eta_0}{2}B \,.
    \end{align*}
    The first term in the last line is proportional to $1 / \eta_{n+1}$, therefore using the definition of $\eta_t$, Jensen's inequality and $\ell_{tA_t}^2 \leq \ell_{tA_t}$, the regret can be bounded as
    \begin{align*}
        &R_n \leq h_2(n) + \frac{\eta_0}{2}B + \brackets{\frac{\sqrt{2}h_1}{\eta_0} + 2\eta_0}\\
        &\qquad\times \sqrt{1 + \half \E\left[ \sum_{t=1}^n \frac{\ell_{tA_t}}{P_{tA_t}^2\nabla^2 f_t(P_{tA_t})}\right]}\,.
    \end{align*}
    The first bound in the theorem follows from
    \begin{align*}
        &\E\left[\sum_{t=1}^n \frac{\ell_{tA_t}}{P_{tA_t}^2\nabla^2 f_t(P_{tA_t})}\right] \\
        &\qquad\leq B\E\left[\sum_{t=1}^n \brackets{\ell_{tA_t} - \ell_{t1} + \ell_{t1}}\right]
        = B R_n + B L_{n1}\,
    \end{align*}
    and then from choosing $\eta_0 = \sqrt{h_1} / 2^{1/4}$ and solving the resulting quadratic equation with respect to $R_n$.
    
    For the second bound, we use $\l_{tA_t} \leq 1$ and the definition of $C$, such that
    \begin{align*}
        &\E\left[\sum_{t=1}^n \frac{\ell_{tA_t}}{P_{tA_t}^2\nabla^2 f_t(P_{tA_t})}\right] \leq \frac{C n}{2}\,.
    \end{align*}
\end{proof}
To prove the corollaries, we need to bound $h_1,\ h_2(n),\ B,$ and $C$:

\begin{lemma}\label{lemma:hybrid_B}
    The Hessian of the hybrid potential in \cref{corollary:hybrid-first-order} is monotonically decreasing, and for $n\geq 3$
    \[
        \frac{1}{p^2 \nabla^2 f_t(p)} \leq \sqrt{k}\log^{1+q}n, \quad \E\left[\frac{1}{P_{tA_t}^2 \nabla^2 f_t(P_{tA_t})}\right] \leq 2\sqrt{k}\,,
    \]
\end{lemma}
\begin{proof}
    For $p \in \interior(\Delta^{k-1})$,
    \begin{align*}
        \nabla^2 f_t(p) \,&= \frac{1}{2p^{3/2}} + \frac{1}{p^2\sqrt{k}\log^{1+q}\max\set{3, t}}\,
    \end{align*}
    is a decreasing function of $p$. It follows that for $n \geq 3$
    \[
        \frac{1}{p^2\nabla^2 f_t(p)} \leq \sqrt{k}\log^{1+q}n\,.
    \]
    Moreover, 
    \begin{align*}
        &\sup_{t, P_{t}\in \cA_t} \EE{\frac{1}{P_{tA_t}^2\nabla^2 f_t(P_{tA_t})}}\\ &\qquad\qquad\leq \sup_{t, P_{t}\in\Delta^{k-1}} \EE{\frac{2}{\sqrt{P_{tA_t}}}} = 2\sqrt{k}\,. \qedhere
    \end{align*}
\end{proof}

\begin{lemma}\label{lemma:v_n-bound}
    Under the conditions of \cref{corollary:hybrid-first-order},
    \begin{align*}
        &v_n \leq  \frac{\sqrt{k}}{\eta_{n+1}}\brackets{\frac{4}{3} + \frac{2}{q}} + \frac{5.5k}{\eta_0}\sqrt{1 + 9k^{3/2}\log^{1+q}(9k^{3/2})} \\
        &\qquad +\frac{3.7\sqrt{k}}{\eta_0}\sqrt{1 + 3\sqrt{k}\log^{1+q}3} + 2k\log n\,.
    \end{align*}
\end{lemma}
\begin{proof}
    Due to the chopped simplex and the factorised potential, we have (recall the definition in \cref{thm:ftrl} and 
    \ifsup
    use \cref{lemma:integral_bound} for the last inequality)
    \else
    bound the second sum with the integral, as shown in the supplementary material)
    \fi
    \begin{align*}
        &v_n = \sum_{t=1}^n d_t (g_t + (t-1)) \\
        &\qquad\leq \sum_{t=1}^n \frac{2k}{t^2} \brackets{\frac{1}{\eta_{t}} \sup_{p\in [1/t, 1]} |\nabla f_t(p)|  + (t-1)}\\
        &\qquad\leq \sum_{t=1}^n \frac{2k}{t^2} \brackets{\frac{1}{\eta_{t}} \sup_{p\in [1/t, 1]} |\nabla f_t(p)|} + 2k \log n \,.
    \end{align*}
    For $p \in [1/t, 1]$ the gradient is bounded as
    \begin{align*}
        &|\nabla f_t(p)| \leq \frac{1}{\sqrt{p}} + \frac{1}{p \sqrt{k}\log^{1+q} \max\set{3, t}}\\
        &\qquad \leq \sqrt{t} + \frac{t}{\sqrt{k}\log^{1+q}\max\set{3, t}} \,.
    \end{align*}
    Therefore, the corresponding sum in $v_n$ converges. By a straightforward calculation
    \ifsup
    (as shown in \cref{lemma:log-sqrt-sum}),
    \else
    (see the supplementary material),
    \fi
    the Hessian is bounded as in \cref{lemma:hybrid_B}),
    \begin{align*}
        &v_n \leq  \frac{\sqrt{k}}{\eta_{n+1}}\brackets{\frac{4}{3} + \frac{2}{q}} + \frac{5.5k}{\eta_0}\sqrt{1 + 9k^{3/2}\log^{1+q}(9k^{3/2})} \\
        &\qquad +\frac{3.7\sqrt{k}}{\eta_0}\sqrt{1 + 3\sqrt{k}\log^{1+q}3} + 2k\log n\,. \qedhere
    \end{align*}
\end{proof}

\begin{lemma}\label{lemma:hybrid-potential-diff}
    Under the conditions of \cref{corollary:hybrid-first-order},
    \begin{align*}
        &\min_{p\in\cA_{n+1}}\brackets{F_{n+1}(p) + \| p - e_1 \|_1 n} - F_1(P_1) \\
        &\qquad+ \sum_{t=1}^n\brackets{F_{t}(P_{t+1}) - F_{t+1}(P_{t+1})} \\
        &\qquad\leq \frac{\sqrt{k}}{\eta_{n+1}}\brackets{3 + \frac{1}{q}} + k + \frac{\sqrt{k}}{3\eta_0}\sqrt{1 + 3\sqrt{k}\log^{1+q} 3}\,.
   \end{align*}
\end{lemma}
\begin{proof}
    The potential is a mixture of the INF and the log-barrier parts, $F_t(p) = -\frac{2}{\eta_t}\sum_i \sqrt{p_i} - \frac{\alpha_t}{\eta_t} \sum_i \log p_i$ with $\alpha_t = 1 / (\sqrt{k}\log^{1+q} \max\set{3, t})$.
    
    To control the contribution of the INF term, first notice that the INF part of $F_{n+1}(p)$ is negative. Moreover,
    \begin{align*}
        &\brackets{-\frac{2}{\eta_t} + \frac{2}{\eta_{t+1}}}\sum_{i=1}^k \sqrt{P_{t+1,i}} \leq 2\sqrt{k}\brackets{\frac{1}{\eta_{t+1}} - \frac{1}{\eta_{t}}} \,.
    \end{align*}
    Summing with the INF part of $-F_1(P_1)$ and telescoping shows that it contributes at most $2\sqrt{k} / \eta_{n+1}$ to the sum.
    
    For log-barrier, suppose $\alpha_t / \eta_t  \leq \alpha_{t+1}/\eta_{t+1}$. Then
    \begin{align*}
    \left(-\frac{\alpha_t}{\eta_t} + \frac{\alpha_{t+1}}{\eta_{t+1}}\right) \sum_{i=1}^k \log(P_{t+1, i}) 
    \leq 0\,.
    \end{align*}
    Now suppose that $\alpha_t / \eta_t > \alpha_{t+1}/\eta_{t+1}$. For $t\geq 3$, as $P_{t+1} \in \cA_{t+1}$,
    \begin{align*}
        &\brackets{- \frac{\alpha_t}{\eta_t} + \frac{\alpha_{t+1}}{\eta_{t + 1}}} \sum_{i=1}^k \log(P_{t+1,i}) \\
        &\qquad\qquad\leq \brackets{\frac{\alpha_t}{\eta_t} - \frac{\alpha_{t+1}}{\eta_{t + 1}}} k\log (t+1) \\
        &\qquad\qquad\leq \frac{1}{\eta_t}\brackets{\frac{\log(t+1)}{\log^{1+q}(t)} - \frac{1}{\log^q(t)}} \sqrt{k} \\
        &\qquad\qquad\leq \frac{\sqrt{k}}{\eta_t t \log^{1+q} t} \,.
    \end{align*}
    Summing over $t$ and noting that due to $\alpha_1 = \alpha_2=\alpha_3$ the potential is unchanged,
    \begin{align*}
        &\sum_{t=1}^n \brackets{-\frac{\alpha_t}{\eta_t} + \frac{\alpha_{t+1}}{\eta_{t + 1}}} \sum_{i=1}^k \log w_{t+1, i} \\
         &\qquad\leq \sum_{t=3}^n \frac{\sqrt{k}}{\eta_t t  \log^{1+q} t} \leq \frac{\sqrt{k}}{\eta_{n+1}q} + \frac{\sqrt{k}}{3\eta_3}\,,
    \end{align*}
    where the last inequality 
    \ifsup
    follows from \cref{lemma:integral_bound}, which
    \else
    (shown in the supplementary material) 
    \fi
    essentially compares the sum to the integral of $1 / (t\log^{1+q}t)$ and uses that $1 / \log^q t \leq 1$ for $t \geq 3$. We can further bound $\eta_3$ as
    \[
    \frac{1}{\eta_3} \leq \frac{1}{\eta_0}\sqrt{1 + 3\sqrt{k}\log^{1+q} 3}
    \]
    by using the fact that the Hessian is bounded (see the proof of \cref{lemma:hybrid_B}).
    
    Finally, the log-barrier part of $-F_1(P_1)$ is negative. The log-barrier part of $F_{n+1}(p)$ is bounded by $\sqrt{k} / \eta_{n+1}$ as $ p \in \cA_{n+1}$. Thus,
    \begin{align*}
        &\min_{p\in\cA_{n+1}}F_{n+1}(p) + n \| p - e_1 \|_1 \\ 
        &\qquad\leq \frac{\sqrt{k}}{\eta_{n+1}} + \min_{p\in\cA_{n+1}} n \| p - e_1 \|_1 = \frac{\sqrt{k}}{\eta_{n+1}} + \frac{kn}{n+1}\,.
    \end{align*}
    
    Combining the three bounds and using that $kn / (n+1) \leq k$ concludes the proof.
\end{proof}
\begin{proof}[Proof of \cref{corollary:hybrid-first-order}] 
From \cref{lemma:hybrid_B}, \cref{lemma:v_n-bound} and \cref{lemma:hybrid-potential-diff}, we find
\begin{align*}
    &B \,= \sqrt{k}\log^{1+q}n\,, \,\,
    C \,= 2\sqrt{k}\,, \,\,
    h_1 = \sqrt{k}\brackets{\frac{13}{3} + \frac{3}{q}}\,,\\
    &h_2(n) = 2k \log n +  \frac{5.5.k}{\eta_0}\sqrt{1 + 9k^{3/2}\log^{1+q}(9k^{3/2})}\\
    &\quad + \frac{4.1}{\eta_0}\sqrt{k}\sqrt{1 + 3\sqrt{k}\log^{1+q}3} + k\,. \qedhere
\end{align*}

Now applying \cref{theorem:hessian-first-order} with $\eta_0 = k^{1/4}\sqrt{\frac{13}{3\sqrt{2}} + \frac{3}{\sqrt{2}q}}$ completes the proof. Note that in the big-O notation, we only kept the leading terms that grow with $n$. 
\end{proof}

\begin{proof}[Proof of \cref{corollary:hybrid-first-order-q}] Starting from the end of the previous proof, choosing $q=1$ and upper-bounding the numerical coefficients, we obtain the corollary.
\end{proof}

%% file: variance.tex
The expected regret is just one measure of the performance of an algorithm.
Algorithms with small expected regret may suffer from a large variance. Since the adversarial model is often motivated on the grounds of providing robustness, 
it would be unfortunate if proposed algorithms suffered from high variance. 
Recently, however, it was shown that the variance of Exp3 without exploration is quadratic in the horizon \citep[\S11]{LS19bandit-book}, and a similar result holds for 
Thompson sampling in a Bayesian setting \citep{BS19}. 
Here we generalise these arguments to prove quadratic variance of the regret for a class of algorithms based on FTRL with importance-weighted loss estimators.
This is the worst possible result for bandits with bounded losses. The class of policies covered by our theorem includes INF and Exp3, but not FTRL 
with the log barrier. To keep things simple we restrict ourselves to algorithms of the form
\begin{align*}
P_t = \argmin_{p \in \Delta^{k-1}}  \ip{p, \hat L_{t-1}} + \frac{1}{\eta_n}\sum_{i=1}^k f(p_i)\,,
\end{align*}
where $f$ is convex and $(\eta_n)_{n=1}^\infty$ is a sequence of learning rates.
Note that this corresponds to a sequence of algorithms, each with a fixed learning rate. 

\begin{assumption}
The number of actions is $k = 2$ and $f$ is Legendre with $(0,1) \subseteq \dom(f)$ and $0 \in \partial \dom(f)$.
\end{assumption}

The assumption on the potential is satisfied by all standard potentials for bandits on the probability simplex, including those in \cref{tab:pot}. It allows us to write $P_t$ in a simple form.
Let $g(p) = f(p) + f(1 - p)$, which is convex and Legendre with $\dom(g) = (0,1)$.
Given $x \geq 0$,
\begin{align*}
\argmin_{p \in [0,1]} \left(p x + g(p) \right) = \nabla g^*(-x)\,,
\end{align*}
where we used the fact that for Legendre functions the gradient is invertible and $(\nabla g)^{-1} = \nabla g^*$.
That $g$ is Legendre with $\dom(g) = (0,1)$ also ensures that $\nabla g^*$ is nondecreasing and $\lim_{x\to-\infty} \nabla g^*(x) = 0$ and $\lim_{x\to\infty} \nabla g^*(x) = 1$.
By symmetry, we also have $\nabla g^*(0) = 1/2$. The point is that by the definition of FTRL,
$P_{t1} = \nabla g^*(\eta_n (\hat L_{t-1,2} - \hat L_{t-1,1}))$.

\begin{theorem}\label{thm:var}
Assume $\limsup_{n\to\infty} n \nabla g^*(-a n \eta_n) < \infty$ for all $a > 0$.
Then for all sufficiently large $n$ there exists a bandit for which
$\mathbb{P}(\hat R_n \geq n/4) \geq c$,
where $c > 0$ is a constant that depends on the algorithm, but not the horizon.
\end{theorem}

\begin{corollary}
Under the same conditions as \cref{thm:var} the variance of the regret is $\Var[\hat R_n] = \Omega(n^2)$.
\end{corollary}

\paragraph{Examples}
Suppose $\eta_n = a n^{-1/2}$ for some $a > 0$.
Then the conditions of the theorem are satisfied when $f$ is the negentropy. In this case $\nabla g^*$ is the sigmoid function and the corresponding algorithm is just Exp3.
When $f(p) = -2 \sqrt{p}$ and $x \leq 0$, then
\begin{align*}
\nabla g^*(x) = \frac{1}{2}\left(1 - \sqrt{1 + \frac{4\left(2 \sqrt{1 + x^2} -2 - x^2\right)}{x^4}}\right)\,,
\end{align*}
which satisfies $\limsup_{n\to\infty} \nabla g^*(-a\sqrt{n})n = 1/a^2$. In this sense 1/2-Tsallis entropy with $\eta_n = \Theta(n^{-1/2})$ 
just barely satisfies the conditions. The consequence is that the minimax optimal INF policy proposed by \cite{AB09} has quadratic variance.
The log barrier does not satisfy the conditions and we speculate it is more stable.

\begin{proof}[Proof of \cref{thm:var}]
Assume for simplicity that $4$ is a factor of $n$.
Let $\alpha_n \in [0,1/2]$ be a constant to be tuned subsequently and
consider a bandit defined by
\begin{align*}
\ell_{t1} &= \begin{cases}
\alpha_n & \text{if } t \leq n/2 \\
0 & \text{otherwise}\,.
\end{cases}
&
\ell_{t2} &= \begin{cases}
0 & \text{if } t \leq n/2 \\
1 & \text{otherwise}\,.
\end{cases}
\end{align*}
Clearly the first arm is optimal.  
Let $c_1 > 0$ be a constant such that for all sufficiently large $n$ it holds that
$\nabla g^*(-n \eta_n) \leq c_1 / n$,
which is guaranteed to exist by the assumptions in the theorem.
Then define events
$F_t = \cap_{s=n/2+1}^t \{A_s = 2,\, P_{s1} \leq c_1/n\}$.
On the event $F_n$ the random regret satisfies
\begin{align}
\hat R_n \geq \frac{n}{2} - \frac{\alpha_n n}{2} \geq \frac{n}{4}\,.
\label{eq:F}
\end{align}
The theorem follows by proving that $\Prob{F_n} \geq c$ for all sufficiently large $n$ and constant $c > 0$.
The idea is to show that the estimated loss for the optimal arm after the first $n/2$ rounds is large enough that
the algorithm never plays the optimal arm in the second half of the game with constant probability.

\paragraph{First half dynamics}
The choice of $\alpha_n$ determines the dynamics of the interaction between the algorithm and environment in the first $n/2$ rounds.
Before the main proof we establish some facts about this. Let $\alpha \in [0,1/2]$ and
define $(p_s(\alpha))_{s=0}^n$ inductively by $p_0(\alpha) = 1/2$ and
\begin{align*}
p_{s+1}(\alpha) = \nabla g^*\left(-\eta_n \sum_{u=0}^s \frac{\alpha}{p_u(\alpha)}\right)\,,
\end{align*}
which is chosen so that $P_{t+1,1} = p_s(\alpha)$ whenever $t+1 \leq n/2$ and $T_1(t) = s$.
Here we used the fact that $\hat L_{t2} = 0$ for $t \leq n/2$, which follows from the definition of the bandit.
Let $Q_s(\alpha) = \sum_{u=0}^{s-1} \alpha / p_u(\alpha)$.
Clearly $Q_2(1/2) > 0$ and $Q_s(0) = 0$ for all $s$. 
Furthermore, $Q_s(\alpha)$ is increasing in both $\alpha$ and $s$ and continuous in $\alpha$.
Therefore there exists an $\alpha_{\circ} \in (0,1/2)$ such that $Q_2(1/2) \geq Q_3(\alpha_\circ)$.
Now suppose that $Q_s(1/2) \geq Q_{s+1}(\alpha_{\circ})$. Using the fact that $\nabla g^*$ is increasing,
\begin{align*}
&Q_{s+1}(1/2) 
= Q_s(1/2) + \frac{1}{2} \nabla g^*\left(-\eta_n Q_s(1/2)\right)^{-1} \\
&\geq Q_{s+1}(\alpha_{\circ}) + \alpha_{\circ} \nabla g^*\left(-\eta_n Q_{s+1}(\alpha_{\circ})\right)^{-1} 
= Q_{s+2}(\alpha_{\circ})\,,
\end{align*}
which by induction means that $Q_s(1/2) \geq Q_{s+1}(\alpha_{\circ})$ for all $s \geq 2$.
Notice that $\hat L_{t1} = Q_s(\alpha_n)$ when $T_1(t-1) = s$.

\paragraph{Second half dynamics}
Define threshold $\lambda_n$ by
\begin{align*}
\lambda_n = n + n^2/ (2(n-c_1)) \leq 2 n\,,
\end{align*}
where the latter inequality holds for all sufficiently large $n$.
Let $E$ be the event $E = \{\hat L_{n/2,1} \geq \lambda_n\}$.  
We claim that $\Prob{F_n \mid E} \geq \exp(-c_1/2)$.
Suppose that $t > n/2$ and $E \cap F_t$ occurs. Then 
\begin{align*}
\hat L_{t2} = \sum_{s=n/2+1}^t \frac{1}{P_{s2}} \leq \sum_{s=n/2+1}^t \frac{1}{1 - c_1/n} \leq \frac{n^2}{2(n-c_1)}\,, 
\end{align*}
where the first inequality follows from the definition of $F_t$.
Therefore, since $\hat L_{t1} \geq \hat L_{n/2,1} \geq \lambda_n$, 
\begin{align}
P_{t+1,1} 
&= \nabla g^*(\eta_n(\hat L_{t2} - \hat L_{t1})) \nonumber \\
&\leq \nabla g^*\left(\eta_n\left(\frac{n^2}{2(n-c_1)} - \lambda_n\right)\right) 
\leq \frac{c_1}{n}\,. \label{eq:var1}
\end{align}
Hence $\Prob{F_{t+1} \mid F_t, E} \geq 1 - c_1/n$. Noting that \cref{eq:var1} implies that $P_{n/2+1,1} \leq c_1/n$ shows that $E \subseteq F_{n/2+1}$ and hence by induction
\begin{align}
\Prob{F_n \mid E} \geq \left(1 - \frac{c_1}{n}\right)^{n/2} \geq \exp(-c_1/2)\,.
\label{eq:FE}
\end{align}

\paragraph{Lower bounding $\Prob{E}$}
By \cref{eq:F,eq:FE} it suffices to prove that $\Prob{E}$ is larger than a constant for sufficiently large $n$.
Let $s = \min\{u : Q_u(1/2) \geq \lambda_n\}$, which by our assumptions on $\nabla g^*$ for sufficiently large $n$ is at least $s > 2$ and at most $s \leq n/2$. Then
$Q_s(\alpha_{\circ}) \leq Q_{s-1}(1/2) < \lambda_n \leq Q_s(1/2)$.
By the intermediate value theorem and the continuity of $\alpha \mapsto Q_s(\alpha)$ we may choose 
$\alpha_n \in (\alpha_{\circ}, 1/2]$ such that $Q_s(\alpha_n) = \lambda_n$. 
Now introduce a sequence of independent geometric random variables $(G_u)_{u=0}^s$ with $G_u \in \{1,2,\ldots\}$ and $\E[G_u] = 1/p_u(\alpha)$.
Then by construction,
\begin{align}
\Prob{T_1(n/2) \geq s} = \Prob{\sum_{u=0}^{s-1} G_u \leq \frac{n}{2}}\,. \label{eq:G}
\end{align}
You should think of $G_u$ as the number of rounds before the algorithm plays action $1$ for the $u$th time.
Let
\begin{align*}
\kappa = \min\set{m : \sum_{u=0}^{s-m-1} \frac{1}{p_u(\alpha_n)} \leq \frac{n}{8}}\,.
\end{align*}
Then either $\sum_{u=0}^{s-\kappa-1} 1/p_u(\alpha_n) \leq n/16$ in which case $1/p_{s-\kappa}(\alpha_n) \geq n/16$ or
$\sum_{u=0}^{s-\kappa-1} 1/p_u(\alpha_n) \geq n/16$. Then there exists a constant $c_2 \geq 0$ such that for sufficiently large $n$,
\begin{align*}
p_{s-\kappa}(\alpha_n) 
&= \nabla g^*(-\eta_n Q_{s-\kappa-1}(\alpha_n)) \\
&= \nabla g^*\left(-\eta_n \sum_{u=0}^{s-\kappa-1} \frac{\alpha_n}{p_u(\alpha_n)}\right) \\
&\leq \nabla g^*\left(-\frac{\alpha_{\circ} n\eta_n}{16}\right) 
\leq \frac{c_2}{n}\,.
\end{align*}
Combining the two cases and choosing $c_2 \geq 16$ guarantees that $p_{s-\kappa}(\alpha_n) \leq c_2/n$ for sufficiently large $n$.
Using the fact that $s \mapsto p_s(\alpha_n)$ is decreasing,
\begin{align*}
2 n 
&\geq \lambda_n 
= \sum_{u=0}^{s-1} \frac{\alpha_n}{p_u(\alpha_n)} 
\geq \sum_{u=s-\kappa}^{s-1} \frac{\alpha_{\circ} n}{c_2} 
= \frac{\kappa \alpha_{\circ} n}{c_2}\,. 
\end{align*}
Rearranging shows that $\kappa$ is less than a constant that is independent of $n$. 
By Markov's inequality
\begin{align*}
&\Prob{\sum_{u=0}^{s-\kappa-1} G_u \geq \frac{n}{4}} \\
&\qquad\qquad\leq \Prob{\sum_{u=0}^{s-\kappa-1} G_u \geq 2\sum_{u=0}^{s-\kappa-1} \frac{1}{p_u(\alpha_n)}} 
\leq \frac{1}{2}\,.
\end{align*}
Hence
\begin{align}
\Prob{\sum_{u=0}^{s-\kappa-1} G_u < \frac{n}{4}} \geq \frac{1}{2}\,. \label{eq:G1}
\end{align}
Furthermore,  
\begin{align*}
\frac{\alpha_{\circ}}{p_{s-1}(\alpha_n)} 
&\leq \frac{\alpha_n}{p_{s-1}(\alpha_n)} 
\leq \sum_{u=0}^{s-1} \frac{\alpha_n}{p_u(\alpha_n)} \\
&= Q_s(\alpha_n) = \lambda_n \leq 2 n\,. 
\end{align*}
Therefore, using again that $s\mapsto p_s(\alpha)$ is decreasing,
\begin{align*}
&\Prob{\sum_{u=s-\kappa}^{s-1} G_u \leq \frac{n}{4}} \\
&\quad\geq {n/4 \choose \kappa} p_{s-1}(\alpha_n)^\kappa \left(1 - p_{s-\kappa}(\alpha_n)\right)^{n/4-\kappa} \\
&\quad\geq {n/4 \choose \kappa} \left(\frac{\alpha_{\circ}}{2 n}\right)^\kappa \left(1 - \frac{c_2}{n}\right)^{n/4-\kappa}\,, 
\end{align*}
which for sufficiently large $n$ is larger than a strictly positive constant and the result follows by combining the above with \cref{eq:G,eq:G1}.
\end{proof}

\begin{remark}
We believe the result continues to hold for adaptive learning rates under the
assumption that $\limsup_{t\to\infty} t \nabla g^*(-at \eta_t) < \infty$ for all $a > 0$.
The proof becomes significantly more delicate, however.
\end{remark}

%% file: gap.tex
In this section we consider the case where the adversary chooses an infinite sequence of loss vectors $(\ell_t)_{t=1}^\infty$. The main objective is to prove logarithmic (or better) regret
under the following assumption.

\begin{assumption}\label{ass:linear}
There is a linear separation between the optimal and suboptimal arms:
\begin{align*}
\Delta_i = \liminf_{n\to\infty} (L_{ni} - L_{n1})/n > 0 \quad \text{for all } i > 1\,. 
\end{align*}
\end{assumption}

Note that if $(\ell_t)_{t=1}^\infty$ are independent and identically distributed random vectors, then the above holds almost surely whenever
there is a unique optimal arm. We provide two results in this setting.
The first generalises a known result from stochastic bandits that there exist algorithms for which the asymptotic random regret grows arbitrarily slowly almost surely \citep{CK15c}.

\begin{theorem}\label{thm:asy}
For any nondecreasing function $f : \N \to \N$ with $\lim_{n\to\infty} f(n) = \infty$ there exists
an algorithm such that $\limsup_{n\to\infty} \hat R_n / f(n) < \infty$ almost surely.
\end{theorem}

The algorithm realising the bound in \cref{thm:asy} explores uniformly at random on a set $E$ for which
$\limsup_{n\to\infty} |E \cap [n]| / f(n) \leq 1$ almost surely. 
The reader is warned that the constants hidden by the asymptotics are potentially quite enormous.

Of course this result says nothing about the expected regret, which must be logarithmic for consistent algorithms \citep{LR85}.
The following theorem improves on a result by \cite{SS14} by a factor of $\log(n) / \log\log(n)$. 

\begin{theorem}\label{thm:linear}
There exists an algorithm such that for any adversarial bandit $R_n = \cO(\sqrt{kn})$.
Furthermore, under \cref{ass:linear} it holds that
\begin{align*}
\limsup_{n\to\infty} \frac{R_n}{\log(n)^2 \log\log(n)} < \infty\,. 
\end{align*}
\end{theorem}

The algorithm is INF with enough forced exploration that the loss estimators are guaranteed to be sufficiently accurate to detect a linear separation. The proofs of \cref{thm:asy,thm:linear} use standard concentration results and are given in
\ifsup
\cref{app:asy} and \cref{app:linear} respectively.
\else
the supplementary material.
\fi

%% file: discussion.tex
Despite the relatively long history and extensive research, many open questions exist about $k$-armed adversarial bandits. Perhaps the most exciting question is the existence/nature of a genuinely instance-optimal algorithm. The work by \cite{ZiSe18} suggests the possibility of an algorithm for which $R_n = \cO(\sqrt{kn})$ and $R_n = \cO(\sum_{i : \Delta_i > 0} \log(n) / \Delta_i)$, where $\Delta_i = \frac{1}{n} \sum_{t=1}^n (\ell_{ti} - \ell_{t1})$ is the empirical gap between the arms. In fact, one could hope for a little more. For stochastic Bernoulli bandits with means $(\theta_i)_{i=1}^k$, the KL-UCB algorithm by \cite{CGMMS13} satisfies $R_n=\cO(\sum_{i:\Delta_i>0} \Delta_i \log(n) / d(\theta_i, \theta^*))$ where $d(\theta_i, \theta_1)$ is the relative entropy between Bernoulli distributions with bias $\theta_i$ and $\theta_1$ respectively. We are not aware of a lower bound proving that such a result is not possible for adversarial bandits with $\theta_i = \frac{1}{n} \sum_{t=1}^n \ell_{ti}$. At present it is not clear whether or not our modified algorithm from \cref{corollary:hybrid-first-order} retains the logarithmic regret in the stochastic setting, both because we use an adaptive learning rate and a hybrid potential. Finally, it is known that sub-exponential tail bounds are incompatible with logarithmic regret in the stochastic setting \citep{AuMuSze09}, but by appropriately tuning the confidence intervals it is straightforward to prove the variance is linear in $n$, which is optimal.
Missing is an adaptation of INF that enjoys (a) minimax regret, (b) logarithmic regret in the stochastic setting and (c) linear variance.

%% file: acknowledgements.tex
\vspace{-1pt}This work was supported by the Gatsby Charitable Foundation. The authors thank Haipeng Luo for spotting  an error in the earlier version of the manuscript.

%% file: tech.tex
\begin{proof}[Proof of \cref{lemma:sqrt_sum_bound}]
The result is immediate if $\sum_{t=1}^n x_t < B$. Otherwise 
let $t_{\circ} = \min\{t : \sum_{s=1}^{t-1} x_s \geq B\}$.
Then
\begin{align*}
\sum_{t=1}^n \frac{x_t}{\sqrt{1 + \sum_{s=1}^{t-1} x_s}} 
&\leq B + \sum_{t=t_{\circ}}^n \frac{x_t}{\sqrt{1 + \frac{1}{2} \sum_{s=1}^t x_s}}\,.
\end{align*}
Next let $f(t) = x_{\ceil{t}}$ and $F(t) = \int^{t}_0 f(s) ds$. Then
\begin{align*}
&\sum_{t=t_{\circ}}^n \frac{x_t}{\sqrt{1 + \frac{1}{2} \sum_{s=1}^t x_s}}
\leq \int^n_0 \frac{f(t)}{\sqrt{1 + F(t) / 2}} dt \\
&\quad\leq 4 \sqrt{1 + F(n)/2} 
= 4 \sqrt{1 + \frac{1}{2} \sum_{t=1}^n x_t}\,.
\end{align*}
The result follows from the previous two displays.
\end{proof}

\begin{lemma}\label{lemma:integral_bound}
Let $(x_t)_{t=t_{\circ}}^{\infty}$ be a sequence of positive non-decreasing elements, and $f(x)$ be a continuous non-increasing functions such that $f(t)=x_t,\ t\geq t_{\circ}$. Then
\[
\sum_{t=t_{\circ}}^n x_t \leq x_{t_{\circ}} + \int_{t_{\circ}}^n f(t)dt.
\]
\end{lemma}

\begin{proof}
Follows from the geometric definition of the Riemann integral.
\end{proof}

\begin{lemma}\label{lemma:log-sqrt-sum}
    If a non-increasing learning rate $\eta_t$ is such that $1 / \eta_t \leq \sqrt{1 + t\sqrt{k}\log^{1+q}t} / \eta_0$, then
    \begin{align*}
        &\sum_{t=1}^n \frac{2k}{\eta_t t^2} \brackets{\sqrt{t} + \frac{t}{\sqrt{k}\log^{1+q}\max\set{3, t}}} \\
        &\quad \leq \frac{\sqrt{k}}{\eta_{n+1}}\brackets{\frac{4}{3} + \frac{2}{q}} + \frac{5.5 k}{\eta_0}\sqrt{1 + 9k^{3/2}\log^{1+q}(9k^{3/2})} \\
        &\quad +\frac{3.7\sqrt{k}}{{\eta_0}}\sqrt{1 + 3\sqrt{k}\log^{1+q}3}\,.
    \end{align*}
\end{lemma}
\begin{proof}
    Consider the first part of the sum. Splitting it at $t_{\circ} = 9k$, applying \cref{lemma:integral_bound} and using $\sum_{t=1}^{\infty} 1/t^{3/2} \leq 2.7$, 
    \begin{align*}
        &\sum_{t=1}^n \frac{2k}{\eta_t t^{3/2}} \leq \sum_{t=1}^{t_{\circ}} \frac{2k}{\eta_{t_{\circ}} t^{3/2}} + \sum_{t=t_{\circ}}^n \frac{2k}{\eta_t t^{3/2}}\\ &\qquad \leq \frac{5.4k}{\eta_{t_{\circ}}} + \frac{2k}{\eta_{t_{\circ}}t_{\circ}^{3/2}} + \frac{4k}{\eta_{n+1}t_{\circ}^{1/2}}\\
        &\qquad \leq \frac{5.5 k}{\eta_0}\sqrt{1 + 9k^{3/2}\log^{1+q}(9k^{3/2})} + \frac{4\sqrt{k}}{3\eta_{n+1}}\,.
    \end{align*}
    
    For the second sum, using the integral of $1 / (x\log^{1+q}x)$ from $t_{\circ} = 3$ to $\infty$ in \cref{lemma:integral_bound} and $\log n \geq \log 3 > 1$,
    \begin{align*}
        &\sum_{t=1}^n \frac{2\sqrt{k}}{\eta_t t \log^{1+q}\max\set{3, t}} \leq \frac{2\sqrt{k}}{\eta_3}\brackets{1 + \frac{1}{2} + \frac{1}{3}} \\
        &\qquad\qquad + \int_{t=3}^n \frac{2\sqrt{k}}{\eta_{n+1} t \log^{1+q}\max\set{3, t}}\\ 
        &\qquad\qquad \leq\frac{3.7\sqrt{k}}{\eta_3} + \frac{2\sqrt{k}}{\eta_{n+1}q}.
    \end{align*}
    
    Combining the two completes the proof.
\end{proof}

%% file: fo-appendix.tex
\begin{proof}[Proof of \cref{corollary:hybrid-known-horizon}]
    First, we repeat the proof of \cref{thm:ftrl} and note that for a time-independent $\cA_t$ we have $v_n = 0$. Therefore, the regret is bounded as
    \begin{align*}
        R_n 
        &\leq \EE{\sum_{t=1}^n \dotprod{P_t - P_{t+1}}{\hat \ell_t} - D_{F_t}(P_{t+1}, P_t)} \\
        &\quad+ \EE{\min_{p \in \cA_{n+1}} F_{n+1}(p)+k - F_1(P_1)} \\
        &\quad+ \EE{\sum_{t=1}^n\brackets{F_{t}(P_{t+1}) - F_{t+1}(P_{t+1})}} \,.
    \end{align*}
    
    Now we repeat the proof of \cref{lemma:v_n-bound} to find $h_1, h_2(n)$. The argument for the INF term in the regularizer is unchanged.
    
    For the log-barrier term, due to the non-decreasing learning rate, at each step
    \[
    \brackets{\frac{1}{\sqrt{k} \eta_{t} \log n} - \frac{1}{\sqrt{k} \eta_{t+1} \log n}}(-\log P_{ti}) \leq 0.
    \]
    Therefore, the only contribution from the log-barrier is from $F_{n+1}(\tilde p) \leq \sqrt{k} / \eta_{n+1}$.
    
    Consequently,
    \begin{align*}
        &F_{n+1}(\tilde p) - F_1(P_1) + \sum_{t=1}^n\brackets{F_{t}(P_{t+1}) - F_{t+1}(P_{t+1})}  \\
        &\qquad\qquad\leq \frac{3\sqrt{k}}{\eta_{n+1}}\,,
   \end{align*}
   and thus $h_1 = 3\sqrt{k},\ h_2 = k$.
   
   Repeating the calculation for the Hessian (essentially for $q=0$), we have that $B = \sqrt{k}\log n$ and $C = 2\sqrt{k}$.
   
   Now using the general bound developed in \cref{theorem:hessian-first-order} with $\eta_0 = k^{1/4}\sqrt{3}/2^{1/4}$, we obtain the statement of the corollary.
    
\end{proof}

%% file: gap-app.tex
Define $\tau(m) = \min\{t : f(t) = m\}$.
We assume without loss of generality that $f(1) = 1$ and that $f$ grows sufficiently slowly so that
\begin{align}
\sum_{m=1}^\infty \sum_{j=m+1}^\infty \frac{\tau(m)}{\tau(j)} < \infty\,, \label{eq:exp-sum} 
\end{align}
Then let $(E_m)_{m=1}^\infty$ be
an infinite sequence of random variables with $E_m$ uniformly distributed on $\{1,\ldots,\tau(m)\} \setminus \{E_1,\ldots,E_{m-1}\}$. 
Let $E = \{E_1,\ldots\}$ be the set of steps on which exploration occurs and 
\begin{align*}
\hat \theta_{mi} = \frac{k}{m} \sum_{j=1}^m \ell_{E_j i} A_{E_j i}\,.
\end{align*}
Then in rounds $t \in E$ the algorithm explores uniformly over all actions.
In rounds $t \notin E$ the algorithm chooses
\begin{align*}
A_t = \argmin_{i \in [k]} \hat \theta_{i f(t-1)}\,.
\end{align*}
Let $\kappa = \max\{t : \hat \theta_{m1} \geq \min_{i > 1} \hat \theta_{mi}\}$.
Then the regret can be decomposed by
\begin{align*}
\hat R_n 
&= \sum_{t=1}^n (\ell_{tA_t} - \ell_{t1}) \leq \kappa + \sum_{t=1}^n \one{t \in E}\,. 
\end{align*}
The result follows by showing that $\kappa$ is almost surely finite and that
\begin{align}
\limsup_{n\to\infty} \frac{1}{f(n)} \sum_{t=1}^n \one{t \in E} \leq 1 \,\, a.s.
\label{eq:gap-sum}
\end{align}
To show \cref{eq:gap-sum},
\begin{align*}
&\Prob{\sum_{t=1}^{\tau(m)} \one{t \in E} \geq m + 1}  \\
&\qquad\leq \sum_{j=m+1}^\infty \Prob{E_j \leq \tau(m)} 
\leq \sum_{j=m+1}^\infty \frac{\tau(m)}{\tau(j)} \,. 
\end{align*}
By Borel-Cantelli and \cref{eq:exp-sum},
\begin{align*}
\limsup_{m\to\infty} \frac{1}{m} \sum_{t=1}^{\tau(m)} \one{t \in E} \leq 1 \,\, a.s.
\end{align*}
Therefore
\begin{align*}
&\limsup_{n\to\infty} \frac{1}{f(n)} \sum_{t=1}^n \one{t \in E} \\
&\qquad\leq \limsup_{n\to\infty} \frac{\sum_{t=1}^{\tau(f(n)+1)} \one{t \in E} }{f(\tau(f(n)+1)) - 1} 
\leq 1 \,\, a.s. 
\end{align*}
For the first part let $X_{mi} = k A_{E_m i} \ell_{E_mi}$ and $\cG_m = \sigma(E_1,\ldots,E_m)$.
Then
\begin{align*}
&\E[X_{mi} \mid \cG_{m-1}] 
= \E\left[k A_{E_m i} \ell_{E_m A_{E_m}} \mid \cG_{m-1}\right] \\
&= \frac{1}{\tau(m) - m + 1} \sum_{t=1}^{\tau(m)} \ell_{ti} \one{t \notin E_1,\ldots,E_{m-1}} \\
&= \frac{1}{\tau(m)} \sum_{t=1}^{\tau(m)} \ell_{ti} + O\left(\frac{m}{\tau(m)}\right) \,.
\end{align*}
Now fix an $i > 1$ and let $\hat \Delta_m = X_{mi} - X_{m1}$. By the previous display,
$\lim_{m\to\infty} \E[\hat \Delta_m \mid \cG_{m-1}] = \Delta_i$ almost surely. Since $\hat\Delta_m$ is bounded, Chow's strong law of large numbers for martingales \citep{Cho67} shows that
\begin{align*}
&\lim_{m\to\infty} (\hat \theta_{mi} - \hat \theta_{m1})
= \lim_{m\to\infty} \frac{1}{m} \sum_{j=1}^m \hat \Delta_j  \\
&= \Delta_i + \lim_{m\to\infty} \frac{1}{m} \sum_{j=1}^m (\hat \Delta_j - \E[\hat \Delta_j \mid \cG_{j-1}]) = \Delta_i \,\, a.s.
\end{align*}
The result follows because $\Delta_i > 0$ by assumption.

%% file: L1-gap.tex
Let $F(p) = -2 \sum_{i=1}^k \sqrt{p_i}$ and
consider the modification of INF that chooses
\begin{align*}
\tilde P_t = \argmin_{p \in \Delta^{k-1}}  \ip{p, \hat L_{t-1}} + \frac{F(p)}{\eta_t}
\end{align*}
and $P_t = (1 - \gamma_t) \tilde P_t + \gamma_t \ones / k$ where $(\gamma_t)_{t=1}^\infty$ and $(\eta_t)_{t=1}^\infty$ are appropriately 
tuned sequences of exploration and learning rates. Define
\begin{align*}
g(n) = \frac{1}{n^2} \sum_{t=1}^n \frac{1}{\gamma_t}\,.
\end{align*}

\begin{lemma}\label{lem:diff}
Suppose that $\hat L_{t-1,i} > \hat L_{t-1,1}$. Then
\begin{align*}
\tilde P_{ti} \leq \frac{1}{\eta_t^2(L_{t-1,i} - L_{t-1,1})^2}\,.
\end{align*}
\end{lemma}

\begin{proof}
Straightforward calculus shows that
\begin{align*}
\tilde P_{ti} = \frac{1}{\eta_t^2(\lambda + \hat L_{t-1,i})^2}\,,
\end{align*}
where $\lambda \in \R$ is the unique value such that $\tilde P_t \in \Delta^{k-1}$. 
Clearly $\lambda > -\hat L_{t-1,1}$ and the result follows. 
\end{proof}

\begin{lemma}\label{lem:mg}
Suppose that $g(n) = o(1/\log(n))$ and let 
\begin{align*}
\tau_i = \max\{t : \hat L_{ti} - \hat L_{t1} \leq t\Delta_i / 2\}\,. 
\end{align*}
Then $\E[\tau_i] < \infty$.
\end{lemma}

\begin{proof}
Define sequence of random variables by
\begin{align*}
M_t = \hat L_{ti} - L_{ti} + L_{t1} - \hat L_{t1}\,,
\end{align*}
which is a martingale adapted to $(\cF_t)_{t=1}^\infty$ with $M_0 = 0$ and
\begin{align*}
\E[(M_{t+1} - M_t)^2 \mid \cF_t] 
&\leq 2 \E\left[\hat \ell_{t+1,i}^2 + \hat \ell_{t+1,1}^2 \,\bigg|\, \cF_t\right] 
\leq \frac{4}{\gamma_t}\,.
\end{align*}
By a finite-time version of the law of the iterated logarithm \citep{Bal14} it holds with probability at least $1 - \delta$ that
\begin{align}
\frac{|M_t|}{t} \leq c \sqrt{g(t) \log \left(\frac{\log(t)}{\delta}\right)}\,.
\label{eq:lil}
\end{align}
Then define random variable $\Lambda$ to be the smallest value such that $\Lambda \geq 1$ and
\begin{align*}
\frac{|M_t|}{t} \leq c \sqrt{g(t) \log\left(\Lambda \log(t)\right)} \quad \text{ for all } t\,. 
\end{align*}
By \cref{eq:lil}, $\Prob{\Lambda \geq x} \leq 1/x$ for all $x \geq 1$. 
Let $h : \R \to \R$ be a strictly decreasing function such that $g(n) \leq h(n)$ and $h(n) = o(\log(n)^{-1})$.
Using the definition of $M_t$, \cref{ass:linear} and by
inverting the above display,
\begin{align*}
\tau_i \leq h^{-1}\left(\frac{c_1}{\log(\Lambda)}\right) + c_2\,,
\end{align*}
where $c_1, c_2$ are constants that depend on the loss sequence, but not the horizon.
Hence
\begin{align*}
\E[\tau_i] 
&\leq \int^\infty_0 \Prob{h^{-1}\left(\frac{c_1}{\log(\Lambda)}\right) \geq x - c_2} dx \\
&= \int^\infty_0 \Prob{\Lambda \geq \exp(c_1 h(x - c_2))} dx \\
&\leq \int^\infty_0 \min\{1, \exp(-c_1 h(x - c_2))\} dx \\
&< \infty\,. \qedhere
\end{align*}
\end{proof}

\begin{proof}[Proof of \cref{thm:linear}]
Suppose that $\eta_t$ and $\gamma_t$ are defined by
\begin{align*}
\eta_t &= \sqrt{1/t} &
\gamma_t &= \frac{\log(t) \log\log(t)}{t}\,. 
\end{align*}
That $R_n = \cO(\sqrt{nk})$ follows from the standard analysis of INF with adaptive learning rates \citep{ZiSe18} and the observation that the exploration only contributes a lower order
term of order
\begin{align*}
\sum_{t=1}^n \gamma_t = o(\sqrt{n})\,.
\end{align*}
For the second part. Given $i > 1$ define random time
\begin{align*}
\tau_i = \max\{t : \hat L_{t-1,i} - \hat L_{t-1,1} \leq t \Delta_i / 2 \}\,. 
\end{align*}
Then let $\tau = \max_{i > 1} \tau_i$. By \cref{lem:diff}, for $t \geq \tau$ the definition of the algorithm ensures that $P_{ti} \leq 4/(t \Delta_i^2)$ for all $i > 1$.
Decomposing the regret,
\begin{align*}
R_n 
&= \E\left[\sum_{t=1}^n (\ell_{tA_t} - \ell_{t1})\right] 
= \E\left[\sum_{t=1}^n \ip{P_t - e_1, \ell_t}\right] \\
&= \E\left[\sum_{t=1}^n \ip{\tilde P_t - e_1, \ell_t}\right] + \E\left[\sum_{t=1}^n \gamma_t \ip{\ones/k - \tilde P_t, \ell_t}\right] \\
&\leq \E\left[\sum_{t=1}^n \ip{\tilde P_t - e_1, \ell_t}\right] + \sum_{t=1}^n \gamma_t \\
&\leq \E\left[\sum_{t=\tau+1}^n \ip{\tilde P_t - e_1, \ell_t}\right] + \E[\tau] + \sum_{t=1}^n \gamma_t \\
&\leq \E[\tau] + \sum_{t=1}^n \gamma_t + \cO(\log(n)) \\
&\leq \sum_{i=2}^k \E[\tau_i] + \sum_{t=1}^n \gamma_t + \cO(\log(n))\,.
\end{align*}
The result follows from \cref{lem:mg} and the fact that
\begin{equation*}
\sum_{t=1}^n \gamma_t = \cO\left(\log(n)^2 \log\log(n)\right)\,. \qedhere
\end{equation*}
\end{proof}